\documentclass[preprint]{alt2027} 

\usepackage{preamble}

\title{Approximating Optimal Policies in Heterogeneous Asynchronous Multi-player Bandits}

\usepackage{times}

\altauthor{%
 \Name{Sacha Arroues-Paykin} \Email{sacha.arroues@polytechnique.edu }\\
 \addr Ecole Polytechnique
  \AND
  \Name{Etienne Boursier} \Email{etienne.boursier@inria.fr}\\
 \addr INRIA, LMO, Université Paris-Saclay%
 \AND
 \Name{Vianney Perchet} \Email{vianney.perchet@normalesup.org}\\
 \addr CREST, ENSAE%
}

\begin{document}

\setcounter{tocdepth}{3}

\doparttoc 
\faketableofcontents 

\maketitle

\begin{abstract}
Motivated by applications to cognitive radio networks, multi-player multi-armed bandits (MMAB) have attracted considerable attention in recent years. The asynchronous variant, in which players are active only at a fraction of timesteps, better reflects real-world decentralized networks, where devices are not permanently active. Despite its practical relevance, this setting remains relatively understudied. Existing work on asynchronous MMAB has focused on homogeneous players, who share the same activation probability, a setting in which the optimal centralized allocation can be easily characterized. 
In contrast, efficiently computing the optimal allocation in the heterogeneous setting, where players have different activation probabilities, has remained an open problem. In this work, we investigate this problem and first establish its NP-hardness. We then turn to the greedy algorithm, which runs in polynomial time, and prove that it achieves a $\frac{1}{e}$-approximation of the optimal allocation. Furthermore, when players are processed in decreasing order of their activation probabilities, the approximation guarantee improves to $\frac{1}{\sqrt{e}}$. 
Overall, our results establish the greedy algorithm as an efficient approximation to the optimal allocation, providing a practical oracle baseline for the design of online algorithms for heterogeneous asynchronous multi-player bandits.
\end{abstract}


\section{Introduction}

Multi-player multi-armed bandits (MMAB) provide a natural framework for modeling decentralized resource allocation in cognitive radio networks, where multiple devices compete for access to a shared set of communication channels. 
This problem extends the classical multi-armed bandit setting, in which a single player repeatedly selects an arm (e.g., a communication channel or frequency) from a finite set of options and receives a reward indicating whether a packet has been successfully transmitted. 
In the multi-player setting, each player independently selects an arm at every timestep. A player receives a reward only if they are the sole player to select that arm; otherwise, a collision occurs due to interference between players, and the reward is zero.

Originally introduced by \citet{liu2010,anandkumar2010}, the standard MMAB problem has since attracted substantial interest in the literature \citep{rosenski2016,avner2014,avner2019,besson2018}. We refer the reader to \citet{boursier2024} for a survey of MMAB. 
As first observed by \citet{boursier2019} and further developed by \citet{wang2020}, collisions can also serve as a communication mechanism between players. By exploiting collision information, players can share information about their rewards and implement nearly decentralized strategies that mimic optimal centralized ones.

A key feature enabling---or at least facilitating---such communication through collisions is synchronization: all players are assumed to act simultaneously at every timestep. This assumption, however, is not representative of many real-world communication networks, where devices are not only decentralized but also asynchronous and heterogeneous. This has motivated growing interest in asynchronous multi-player multi-armed bandits (AMMAB), starting with \citet{bonnefoi2017} and followed by, among others, \citet{dakdouk} and \citet{richard2024}.

In the asynchronous setting, only a subset of players is active at each timestep, with each player's activity governed by an independent Bernoulli random variable. Active players can then select an arm of their choice. This additional source of uncertainty makes AMMAB substantially more challenging, both in terms of communication between players and in terms of the underlying learning objective, even for a centralized learner. 
Indeed, whereas the standard MMAB problem focuses on identifying the best $N$ arms, where $N$ is the number of players, and assigning one arm to each player, AMMAB involves a more complex resource-allocation problem. Players must be assigned to arms so as to balance the \textit{occupancy} across them, where the occupancy of an arm is determined by the probability that exactly one player assigned to it is active at the next timestep.

Even the optimal allocation for a centralized agent with full knowledge of the problem parameters is therefore not straightforward to characterize. \citet{dakdouk} showed that, when players are homogeneous---that is, when they all have the same activation probability---a simple greedy algorithm yields an optimal allocation. This result subsequently paved the way for the development of no-regret algorithms for AMMAB with homogeneous players \citep{dakdouk,richard2024}. 
In practical decentralized networks, such as the Internet of Things (IoT), devices can have widely varying characteristics and, consequently, substantially different activation probabilities \citep{bonnefoi2017}. Yet, the optimal offline allocation in this heterogeneous setting has remained unknown. In particular, \citet{dakdouk} conjectured that the corresponding optimization problem might be NP-hard. In this work, we investigate this open problem, motivated by its relevance to AMMAB with heterogeneous players.

\paragraph{Contributions.} We first confirm the conjecture of \citet{dakdouk} by showing in \Cref{sec:NPhard} that the heterogeneous AMMAB allocation problem is NP-hard in full generality. Our proof relies on a reduction from the $3$-partition problem.

Motivated by this hardness result, we then turn to polynomial-time algorithms that provide provable approximation guarantees for the optimal allocation. Focusing on the greedy algorithm, which is already known to be optimal in the homogeneous case \citep{dakdouk}, we show in \Cref{sec:approximation} that it achieves a $\frac{1}{e}$-approximation in the general heterogeneous setting, and that this guarantee is tight.

Interestingly, this approximation ratio holds regardless of the order in which players are processed. When players are instead processed in decreasing order of their activation probabilities, we obtain an improved approximation guarantee of $\frac{1}{\sqrt{e}}$.

Finally, in \Cref{sec:online_and_poa}, we investigate the impact of decentralized decision-making on allocation efficiency through the price of anarchy. Without additional assumptions on the game, strategic equilibria can lead to an arbitrarily large loss in social welfare compared with the optimal centralized allocation. However, we show that this degradation can be bounded under reasonable assumptions on the arm means and the total activation probability of the players.

\section{Preliminaries and notations}\label{sec:preliminaries}
We are interested in the following instance of a centralized agent playing an AMMAB\footnote{We here focus on the offline problem, where the agent has knowledge of the environment parameters $p_i$ and $\theta_k$.}, that we will refer to as \textit{allocation problem}: 
given $N$ players with non-zero activation probabilities $p_1, \dots, p_N <1$ 
and $K$ arms with means $\theta_1, \dots , \theta_K$, the goal is to allocate players to arms in order to maximize the expected payout.

The set of possible allocations is given by $$\mathcal{S} = \{(S_1, \ldots, S_{K}) \mid S_k \subseteq [N] \text{ and } S_k \cap S_{k'} = \emptyset \text{ if } k\neq k' \}.$$ For an allocation in $\mathcal{S}$, $S_{k}$ represents the set of players assigned to arm $k$. Note that some players may be assigned to no arm, or equivalently, be assigned to a ``dummy'' arm of mean $0$.

%

When choosing an allocation, following AMMAB convention on a single timestep:
\begin{itemize}[itemsep=0pt,topsep=5pt]
    \item each player $i$ is active with probability $p_i$, independently of other players;
    \item the payout of arm $k$ is then drawn according to some random variable of expected value~$\theta_k$;
    \item the payout of arm $k$ is only collected when a single player allocated to arm $k$ is active, since multiple active players would lead to a collision event.
\end{itemize}
Consequently, the expected payout of an allocation $S\in \mathcal{S}$ is given by 
\begin{gather*}
    \mu_{\theta}(S) = \sum_{k\in [K]} \theta_kf(S_k) \\
    \text{where} \qquad f(S_k) = \sum_{i\in S_k}\frac{p_i}{1-p_i}\prod_{i\in S_k}(1-p_i).
\end{gather*}
The total expected payout is the sum of the expected payout of each arm $k$: $\theta_k$ is the average payout of the arm, while $f(S_k)$ is the probability that exactly one player is active on arm $k$. For $k\in [K]$, we will throughout refer to $\sum_{i\in S_k}\frac{p_i}{1-p_i}$ as the \textit{load} on arm $k$. An arm with load $l\geq 1$ is referred to as \textit{saturated}. 
In the offline problem, we seek to find, or approximate, the optimal allocation:
\begin{equation}
    \max_{S\in \mathcal{S}} \mu_{\theta}(S) \label{allocation_problem}.
\end{equation}
We assume without loss of generality\footnote{We can indeed always reorder and scale all arms to satisfy this constraint. Furthermore, negative and $0$ mean arms can be ignored in the offline case, as it is optimal to leave them empty.} that $1\geq \theta_1\geq\dots\geq\theta_K>0$. 
Given an allocation $S$, the marginal gain from adding a player $j$ to arm $k$ is exactly

\begin{equation}\label{marginal}
    \begin{split}
       \mu_{\theta}(S_1,\dots,S_k\cup\{j\},\dots,S_K) - \mu_{\theta}(S) & = \theta_k(f(S_k \cup \{j\})-f(S_k)) \\
       & = \theta_k p_j \prod_{i\in S_k}(1-p_i) \left(1-\sum_{i\in S_k}\frac{p_i}{1-p_i}\right).
    \end{split}
\end{equation}
This leads to the natural definition of the greedy algorithm, given by \Cref{alg:greedy} below \citep{dakdouk}. 
It processes players one by one according to some fixed order---or, equivalently, permutation $\sigma$---of the players and allocates each incoming player to the arm that improves the current allocation the most. We assume without loss of generality that the players are processed in the natural order in \Cref{alg:greedy}, i.e., $\sigma = \mathrm{Id}$.

\begin{algorithm2e}
\caption{Greedy algorithm}\label{alg:greedy}

Initialize $S_k \gets \emptyset$ for any $k\in[K]$\;
\For{$n = 1,\ldots, N$}{
\If {$\max_{\ell\in[K]} \theta_{\ell} \prod_{i\in S_{\ell}}(1-p_i) \left(1-\sum_{i\in S_{\ell}}\frac{p_i}{1-p_i}\right) >0$}{
$k \gets \arg\max_{\ell\in[K]} \theta_{\ell} \prod_{i\in S_{\ell}}(1-p_i) \left(1-\sum_{i\in S_{\ell}}\frac{p_i}{1-p_i}\right)$ \;
$S_{k} \gets S_{k}\cup \{n\}$\tcp*{allocates player $n$ to arm $k$}
}}
\Return $(S_1, \ldots, S_K)$
\end{algorithm2e}

A natural variant of the standard greedy algorithm is to process players in decreasing order of their activation probabilities, which \citet{dakdouk} introduces as \textit{descending-order reward greedy} (DORG). Equivalently, it assumes that the players are indexed such that $p_1 \geq \ldots \geq p_N$. DORG can also be viewed as a greedy algorithm that, rather than selecting for the upcoming player the arm that yields the largest improvement in the current allocation, selects at each step the most improving $(\text{player},\text{arm})$ pair to add to the current allocation. The greedy algorithm runs in time $\mathcal{O}((K + N)\log K)$ by maintaining the arms in a max-heap keyed on their current marginal\footnote{With the possibility of keying on marginal and index tuples for tie-breaking.}. 
DORG has additional runtime $\mathcal O(N\,\log N)$ from ordering players beforehand.

\citet{dakdouk} proves that, for some (unknown) ordering of the players, the greedy \Cref{alg:greedy} is optimal. In particular, this means that for homogeneous players, i.e., in the case where $p_i = p$ for all $i$, greedy yields an optimal solution to the allocation problem \eqref{allocation_problem}. Indeed, in that case, the players' order has no influence on the resulting allocation, by symmetry of the players. 
In the heterogeneous case however, finding this precise order can be difficult; and
\citet{dakdouk} conjectures the allocation problem \eqref{allocation_problem} is NP-hard in the general case.

\section{NP-hardness of the allocation problem} \label{sec:NPhard}

In this section, we prove that the allocation problem \eqref{allocation_problem} is NP-hard in general. This reveals a structural difference between homogeneous instances, for which a simple greedy algorithm yields an optimal allocation, and heterogeneous ones. It also motivates the search for polynomial-time approximation oracles carried out in \Cref{sec:approximation}.

We proceed by reduction from the well-known NP-complete 3-partition problem \citep{garey1979} defined as follows.
\begin{defin}[\textsc{3-Partition problem}]\label{def:3partition}
Let $m\geq 1$ and $B\geq 1$ be integers, and let $a_1,\ldots,a_{3m}$ be positive integers satisfying $\frac{B}{4} < a_i < \frac{B}{2}$ for all $i$ and $\sum_{i\leq 3m} a_i = mB$.
The 3-partition problem consists in deciding whether $[3m]$ can be partitioned into $m$ sets $T_1,\ldots,T_m$ such that $\sum_{i\in T_k} a_i = B$ for all $k\in[m]$.
\end{defin}
The constraint $\frac{B}{4}<a_i<\frac{B}{2}$ forces each $T_k$ to contain exactly three elements, hence the $3$-partition name.
\begin{thm}\label{reduction}
3-partition reduces in polynomial time to the following decision version of the allocation problem \eqref{allocation_problem}: given $(p_i)_{i\leq N}$, $(\theta_k)_{k\leq K}$ and a threshold $\tau^\star$, decide whether $\max_{S\in\mathcal{S}}\mu_\theta(S)\geq\tau^\star$. In particular, the allocation problem \eqref{allocation_problem} is NP-hard.
\end{thm}
 
\begin{proof}
Let $(B, a_1,\ldots,a_{3m})$ be an instance of 3-partition, and let $Q = \sum_{i\leq 3m} a_i^2$. We set
\begin{equation*}
    \varepsilon = \frac{1}{2m^3B^3}
\end{equation*}
and construct the following instance of the allocation problem:
\begin{itemize}
    \item $K = m$ arms, all of mean $\theta_k = 1$;
    \item $N = 3m$ players, with activation probabilities $p_i = a_i\varepsilon$ for $i\in[3m]$;
    \item a threshold $\tau^{\star} = mB\varepsilon + (Q - mB^2 - 1)\varepsilon^2$.
\end{itemize}
All these quantities are rationals whose numerators and denominators are polynomial in $m$ and $B$, so the construction is computable in time polynomial in the size of the 3-partition instance. Moreover $p_i < \frac{1}{4m^3B^2} \leq \frac14$, so all $p_i$ are valid probabilities. We show that
\begin{equation}\label{eq:reduction_equiv}
    \max_{S\in\mathcal{S}}\mu_\theta(S) \geq \tau^{\star} \iff (B,a_1,\ldots,a_{3m}) \text{ admits a 3-partition.}
\end{equation}
Let $S\in\mathcal S$ be any allocation. Since $p_i\leq\frac14$, we have $\frac{p_i}{1-p_i}\leq 2p_i$, and the load of any arm satisfies
\begin{align*}
    \sum_{i\in S_k}\frac{p_i}{1-p_i} 
    &\leq 2\varepsilon\sum_{i\leq 3m}a_i \\
    & = 2\varepsilon mB 
    < 1 
\end{align*}
By \Cref{marginal}, the marginal gain of adding an unassigned player to any arm is then strictly positive. Consequently, any optimal allocation uses all players. 
Let $S\in\mathcal{S}$ be an allocation using all players, and denote $\beta_k = \sum_{i\in S_k}a_i$.
For any $i\in S_k$, 
$$1-(\varepsilon \beta_k-p_i) \leq \prod_{j\in S_k, j\neq i}(1-p_j)$$
by union bound, and since $(1-x) \leq e^{-x} \leq 1-x+\frac{x^2}{2}$ for all $x\geq0$, 
 $$\prod_{j\in S_k, j\neq i}(1-p_j) \leq e^{-(\varepsilon \beta_k -p_i)}\leq1-(\varepsilon \beta_k-p_i)+\frac{(\varepsilon \beta_k-p_i)^2}{2}$$
 Multiplying by $p_i$ and summing over $i\in S_k$, we obtain, 
\begin{equation*}
    0 \leq f(S_k) - \Bigl(\varepsilon \beta_k - (\varepsilon \beta_k)^2 + \sum_{i\in S_k}p_i^2\Bigr) \leq \sum_{i\in S_k}p_i\,\frac{(\varepsilon \beta_k-p_i)^2}{2} \leq \frac{(\varepsilon \beta_k)^3}{2}.
\end{equation*}

Finally,

\begin{equation}\label{eq:expansion_arm}
    f(S_k) = \varepsilon\beta_k + \varepsilon^2\Bigl(\sum_{i\in S_k}a_i^2 - \beta_k^2\Bigr) + R_k,
    \qquad\text{with } 0\leq R_k\leq \frac{\varepsilon^3\beta_k^3}{2}.
\end{equation}

Since $\sum_{k}\beta_k = mB$, summing \Cref{eq:expansion_arm} over the $m$ arms yields
\begin{equation}\label{eq:expansion_full}
    \mu_\theta(S) = mB\varepsilon + \Bigl(Q - \sum_{k\leq m}\beta_k^2\Bigr)\varepsilon^2 + R,
\end{equation}
where
\begin{align*}
    0 \leq R &\leq \frac{\varepsilon^3}{2}\sum_{k\leq m}\beta_k^3 \\&\leq \frac{\varepsilon^3}{2} m^3B^3 = \frac{\varepsilon^2}{4}.
\end{align*}

Moreover, since $\sum_k(\beta_k - B) = 0$,
\begin{equation*}
    \sum_{k\leq m}\beta_k^2 = mB^2 + \sum_{k\leq m}(\beta_k-B)^2.
\end{equation*}
As such, 
\begin{itemize}
    \item If $\beta_k = B$ for all $k$, then \Cref{eq:expansion_full} gives
    \begin{equation*}
        \mu_\theta(S) \geq mB\varepsilon + (Q-mB^2)\varepsilon^2 > \tau^{\star}.
    \end{equation*}
    \item Otherwise, the integers $(\beta_k - B)_{k\leq m}$ sum to $0$ and are not all zero, so at least two of them are non-zero and $\sum_k(\beta_k-B)^2\geq 2$. Then \Cref{eq:expansion_full} gives
    \begin{equation*}
        \mu_\theta(S) \leq mB\varepsilon + (Q-mB^2-2)\varepsilon^2 + \frac{\varepsilon^2}{4} < \tau^{\star}.
    \end{equation*}
\end{itemize}
 
If the instance admits a 3-partition $T_1,\ldots,T_m$, then $S = (T_1,\ldots,T_m)$ is an allocation using all players with $\beta_k = B$ for all $k$, and therefore $\mu_\theta(S) > \tau^{\star}$.
 
Conversely, assume $\max_{S\in\mathcal{S}}\mu_\theta(S)\geq\tau^{\star}$. By the above, an optimal allocation must satisfy $\beta_k = B$ for all $k$. The sets $S_1,\ldots,S_m$ then partition $[3m]$ with $\sum_{i\in S_k}a_i = B$ for all $k$, i.e., they form a 3-partition.
 
This proves \Cref{eq:reduction_equiv} and concludes the reduction.
\end{proof}

\section{Approximation oracles for the allocation problem}\label{sec:approximation}

Given the NP-hardness of the general allocation problem \eqref{allocation_problem}, it is reasonable to look for a computationally efficient oracle providing an approximation of the optimal expected payout. In particular, we aim in this section to provide polynomial time $\alpha$-approximation oracles for some $\alpha\in[0,1]$, i.e., algorithms that output $\widehat{S} = (\widehat{S}_1, \ldots, \widehat{S}_K)$ such that for any problem instance $(p_1, \ldots, p_N)$ and $\theta_1, \ldots, \theta_K$, 
\begin{equation*}
    \mu_{\theta}(\widehat{S}) \geq \alpha \max_{S\in\mathcal{S}}\mu_{\theta}(S). 
\end{equation*}
The greedy \Cref{alg:greedy} described in Section~\ref{sec:preliminaries} is a natural candidate, and is already known to be optimal for homogeneous instances.
\subsection{Maximization over matroids}\label{subsec:maximization_over_matroids}

The allocation problem \eqref{allocation_problem} 
can be framed as maximizing a set function over player-arm pairs. 
Indeed, for $A\subseteq [N]\times[K]$, we can define $A_k = \{i\in [N] \mid (i,k)\in A\}$ and $\mu_\theta(A) = \sum_{k \in [K]} \theta_k f(A_k)$.

Problem~\eqref{allocation_problem} can then be seen as a maximization problem over subsets of $E=[N]\times[K]$ that assign at most one arm per player. 
Let us denote $\mathcal I \subseteq 2^E$ the nonempty collection of such possible allocations. The collection $\mathcal I$ has a specific, useful structure: it is a matroid,\footnote{It is even a partition matroid here.} i.e., it satisfies the following three properties
\begin{enumerate}
    \item it contains the empty set: $\emptyset \in \mathcal I$;
    \item \textbf{hereditary property}: for any $A,A'$ with $A'\subset A$, $A\in\mathcal I \implies A' \in\mathcal{I}$;
    \item \textbf{augmentation property}: for any $A,B\in\mathcal{I}$ with $|B|>|A|$, there exists $x\in B\setminus A$ such that $A\cup \{x\}\in\mathcal{I}$.
\end{enumerate}
Maximizing set functions under a matroid constraint has attracted quite some interest in the optimization literature, with a particular focus on the approximation ratio of greedy algorithms. However, existing works typically rely on two key assumptions on the objective function: monotonicity and submodularity. 
A set function $F : 2^E \rightarrow \mathbb R$ is said to be \textit{monotone} if $$A\subseteq B \implies F(A)\leq F(B);$$
and \textit{submodular} if
\begin{equation*}
A\subseteq B \text{ and } e \in E \setminus B \implies
    F(A \cup \{e\})-F(A) \geq F(B \cup \{e\})-F(B).
\end{equation*}

\paragraph{Monotone submodular maximization.}
When the objective function is both monotone and submodular, the greedy algorithm is known to be a $\frac{1}{2}$-approximation under any matroid constraint \citep{fisher1978}. In the case of a stronger, partition matroid constraint, its local greedy variant is also a $\frac{1}{2}$-approximation \citep{goundan2007}.
Finally, \citet{calinescu2011} showed that a continuous relaxation of the greedy algorithm reaches a $(1-\frac{1}{e})$-approximation ratio under matroid constraints.

\paragraph{Non-monotone submodular maximization.}
When the objective function is only submodular, several polynomial-time algorithms have been developed in the literature. In the unconstrained setting, these include local search \citep{feige2011} and the double-greedy algorithm \citep{buchbinder2015}. Under matroid constraints, randomized algorithms using multilinear relaxations have led to approximation ratio guarantees as high as $0.401$ \citep{buchbinder2019, buchbinder2024}.

On the other hand, \citet{gharan2011} established an upper bound for the matroid-constrained setting. More precisely, they showed that no algorithm making only a polynomial number of value queries can guarantee an approximation ratio greater than $0.478$ in full generality.

\paragraph{Allocation problem \eqref{allocation_problem}.} 
By \Cref{marginal}, the marginal gain of adding a player to arm $k$ has the sign of $1-\sum_{i\in A_k}\frac{p_i}{1-p_i}$, and is thus negative as soon as the arm is past saturation. This observation makes our objective function non-monotone, whenever loads larger than $1$ are possible. 

Moreover, adding a player $j$ to arm $k$ changes the marginal gain of any other player $i$ on that arm by
\begin{equation*}
    -\theta_k\, p_i\, p_j \prod_{l\in A_k}(1-p_l)\Bigl(2-\sum_{l\in A_k}\frac{p_l}{1-p_l}\Bigr),
\end{equation*}
which is positive once the load exceeds $2$. This observation makes our objective function non-submodular whenever loads larger than $2$ are possible.

Consequently, the problem considered here does not satisfy the usual assumptions of monotonicity and submodularity; see \Cref{lem:restricted_monotone_submod} in \Cref{annexe_preuve_sous_mod} for further details. 
We therefore rely, in the remainder of this section, on custom arguments that exploit the specific structure of the allocation problem to derive approximation guarantees for the greedy \Cref{alg:greedy}.

\subsection{Approximation bound of greedy}
Let $p_1,\dots,p_N$ and $\theta = (\theta_1, \dots, \theta_K)$ be fixed throughout.
In this subsection, we provide a lower bound on the performance of the greedy \Cref{alg:greedy}, depending on the maximum activation probability among all players.

\begin{thm}\label{thm_1_e}
    For $q = \max_{i\leq N} p_i$, the greedy \Cref{alg:greedy} is a $e^{-q}$-approximation oracle for the allocation problem \eqref{allocation_problem}, regardless of the order in which players are processed. In particular, greedy is an instance-independent $\frac{1}{e}$-approximation.
\end{thm}
\begin{proof}[Sketch of proof.]
We distinguish two cases: one in which the greedy \Cref{alg:greedy} saturates all arms (and potentially has some unassigned players), and one where it does not.
The former is tackled in \Cref{annexe_saturation_complete}, in which we show that said instances yield an even larger approximation ratio, given by $\frac{1}{\lceil \frac{1-q}{q}\rceil e\frac{q}{1-q}\cdot(1-q)^{\lceil \frac{1-q}{q}\rceil}} \geq e^{-q}$.

\medskip

We now assume greedy does not saturate all arms, and therefore in particular assigns all players.
The proof hinges on two main results:
\begin{itemize}
    \item For any $S_k \subseteq [N]$ such that $\sum_{i\in S_k}p_i = m_k$ for a fixed mass $m_k$, the expected payout of arm~$k$ has closed-form upper and lower bounds that solely depend on $q$ and $m_k$. If all players have activation probability at most $q$, then for such $S_k$,
    $$f(S_k) \leq  m_k e^{q-m_k}.$$
    If moreover the mass $m_k$ is at most $1$,
    $$f(S_k) \geq m_k e^{-m_k}$$
    A proof can be found in \Cref{annexe_per-arm_loads}. Informally, the bounds establish that the best way to reach mass $m_k$ is to have few players with large activation probabilities, and the worst when $m_k \leq 1$ is to have a very large number of players with small activation probabilities.

    \item For any per-arm mass profile $c_1,\dots,c_K \in [0,1]$ such that $\sum_k c_k \leq \sum_i p_i$, there exists an allocation $\bar S$, on a set of players potentially different from $p_1,\dots,p_N$, that puts on each arm~$k$ an activation mass of exactly $c_k$, and has a worse expected payout than the allocation given by the greedy \Cref{alg:greedy}.  See \Cref{annexe_load_matching} for a full statement and proof.
\end{itemize}
Putting both of the above statements together, we can construct an allocation $\bar S$ that has worse payout than greedy and matches the capped per-arm mass profile of the optimal allocation $c_k = \min\{m_k^{\text{opt}},1\}$. Then, using the first point, we can bound the payouts of $\bar S$ on each arm.
Specifically, if we denote $q = \max_i p_i$, and $\bar \mu_\theta$ the payout function for the potentially distinct set of players used in $\bar S$, we obtain the following bound
\begin{equation}\label{bound_given_q}
    \mu_{\theta}(S_{\text{Greedy}}) \geq  \bar\mu_{\theta}(\bar S) \geq e^{-q}\mu_{\theta}(S_{\text{opt}}).
\end{equation}
We refer to \Cref{annexe_fin_preuve1/e} for a complete proof.
\end{proof}
\Cref{thm_1_e} establishes an instance-dependent approximation ratio of $e^{-q}$ that depends on $q=\max_{i\in[N]}p_i$. In full generality, we only have $q<1$, hence the general $\frac{1}{e}$-approximation bound.

However, this approximation ratio becomes closer to $1$ as $q$ gets smaller. 
In particular, \Cref{thm_1_e} establishes that greedy is almost optimal as the maximal activation probability goes to $0$, which is close to practice in an IoT context \citep{bonnefoi2017}.

Having now obtained a first result regarding greedy's payout compared to the optimum, one may wonder whether the instance-independent $\frac{1}{e}$ approximation lower bound is tight.
\begin{lem}\label{lem:1/e_tight}
Let $\varepsilon>0$. There exists an instance of the allocation problem, with $K=2$ arms and $N_\varepsilon$ players such that $\frac{\mu_{\theta}(S_{\text{Greedy}})}{\mu_{\theta}(S_{\text{opt}})} \leq \frac{1}{e}+\varepsilon$.
\end{lem}
\begin{proof}[Sketch of proof.]
Let $N\in\mathbb{N}_{\geq 3}$. We construct a two arm instance with $\theta_1 = 1$, $\theta_2 = \frac{1}{3N}$ and activation probabilities: $p_1=\dots=p_{N-1} = \frac{1}{N}$ and $p_N = 1 - \frac{1}{N}$. 
We show in \Cref{annexe_borne_tight} that, on these instances, the greedy algorithm assigns the first $N-1$ players to the first arm and the last player to the second, whereas the optimal allocation makes the opposite assignment.
This leads to an exact ratio of
\begin{equation*}
   \frac{\mu_{\theta}(S_{\text{Greedy}})}{\mu_{\theta}(S_{\text{opt}})} = \frac{(1-\frac{1}{N})^{N-1} + \frac{1-\frac{1}{N}}{3N}}{1-\frac{1}{N}+\frac{(1-\frac{1}{N})^{N-1}}{3N}} \xrightarrow[N\rightarrow \infty]{}\frac{1}{e}. \vspace{-1em}
\end{equation*}
\end{proof}
One therefore cannot hope for a better approximation with the greedy algorithm in all generality. %
However, as seen in the proof of \Cref{lem:1/e_tight}, pathological cases in which the greedy \Cref{alg:greedy} performs the worst seem to be those in which the most active players are seen last, and therefore cannot be assigned to the best arms. 
%

This shortcoming is in particular addressed if greedy processes players in decreasing order of activation probability, i.e., DORG.

\subsection{Approximation bound for DORG}

As a particular instance of the greedy \Cref{alg:greedy}, DORG is already known to be a $\frac{1}{e}$-approximation algorithm, thanks to \Cref{thm_1_e}. 
Particularly motivated by the upper bound examples of \Cref{lem:1/e_tight}, we study in this section whether DORG can lead to an improved approximation ratio, with respect to proceeding in a worst case order of the players.
\begin{thm}\label{thm:DORG_oracle}
DORG is a $\frac{1}{\sqrt e}$-approximation oracle for the allocation problem \eqref{allocation_problem}.
\end{thm}
\begin{proof}[Sketch of proof.]
 We show that DORG and an optimal allocation assign players of activation probability $p > \frac{1}{2}$ the same way. Furthermore, both DORG and optimal allocations have such players as sole occupants of their respective arms. From there, DORG can only differ from said optimal allocation on players with activation probabilities smaller than $\frac{1}{2}$, and on a subset of arms. The instance-dependent bound of \Cref{thm_1_e} then yields a $\frac{1}{\sqrt e}$ approximation ratio on those players. See \Cref{annexe_dorg} for a full proof.
\end{proof}
As the proof sketch above illustrates, the key advantage of DORG lies in how it processes players with activation probabilities larger than $\frac12$. Indeed, the $\frac{1}{\sqrt{e}}$ approximation guarantee holds for any greedy algorithm that processes players with activation probabilities greater than $\frac12$ in decreasing order, regardless of the order in which players with smaller activation probabilities are then processed. In other words, the guarantee does not rely on DORG processing players with small activation probabilities in decreasing order. 
We conjecture that exploiting this additional structure in the analysis could lead to an improved approximation guarantee for DORG.

\section{Online and decentralized considerations}\label{sec:online_and_poa}

The section above proves that we have at hand a deterministic $\alpha$-oracle for some $\alpha\geq\frac{1}{\sqrt e}$, that might depend on the considered instance. Fixing such an $\alpha$, if player activation probabilities are known and arm means are not, and given the payout structure of the allocation problem, we are in the scope of combinatorial bandits with probabilistically triggered arms.
The seminal work of \citet{wangchen2017} then bounds the $(\alpha,1)$-approximation regret of the CUCB algorithm by 
\begin{equation*}
    \text{Reg}_{ \theta,\alpha,1}(T) = \mathcal{O}\left(\sqrt{K\min \{K, N\}T\ln T}\right).
\end{equation*}
We refer to \Cref{annexe_cucb_constants} for further details, and additional distribution dependent bounds.

The direct use of CUCB mentioned above relies on a central learning agent, that can unilaterally assign players to different arms at each round. The AMMAB problem is however mostly motivated by a decentralized application, where each player corresponds to a single device. In the cooperative decentralized case, it is still possible to adapt centralized learning strategies, either via the direct use of communication channels \citep{richard2024}, or via the use of collision information and sophisticated protocols \citep{boursier2019}.

Consequently, typical MMAB works can be adapted to the baseline of the greedy algorithm stated above for deriving decentralized cooperative learning algorithms. We instead focus in the next subsection on the non-cooperative decentralized instance, where players aim to maximize their own expected payout. Here again, the offline problem already provides a lot of interesting insights on this case.

\subsection{Decentralized play and price of anarchy}

In this section, we establish how the absence of a central controller affects the total expected payout in an offline case, via the study of the price of anarchy.

We therefore now let the players choose their arms themselves. Throughout this subsection, the arm means $\theta_1\geq\dots\geq\theta_K>0$ and the activation probabilities $p_1,\dots,p_N$ are fixed and known to all players. We wish to know by how much overall payout is reduced when the allocation emerges from selfish behavior rather than from centralized policies.

We define the decentralized allocation problem as follows: each player $j\leq N$ chooses an arm $s_j\in[K]$. With choices $s=(s_1,\dots,s_N)$, the allocation is $S(s)=(S_1,\dots,S_K)$ with $S_k=\{j \mid\ s_j=k\}$. The expected payout of player $j$ is therefore
\begin{align*}
    u_j(s) &= p_j\,\theta_{s_j}\prod_{i\neq j , s_j=s_i}(1-p_i)\\
    &=p_j \, \theta_{s_j}\frac{z_{s_j}}{1-p_j}
\end{align*}
where $z_k = \prod_{i, s_i=k} (1-p_i)$. 

A profile $s$ is a pure Nash equilibrium if no player gains by switching arms, i.e., if for every player $j$ and every arm $k\neq s_j$,
$$\theta_{s_j}\,\frac{z_{s_j}}{1-p_j}\geq\theta_k\,z_k.$$
Let $\mathcal N$ be the set of pure Nash equilibria (see \Cref{annexe_jeu_potentiel} for a proof that $\mathcal N \neq \emptyset$) and $S_{\text{opt}}$ a centralized optimal allocation. The price of anarchy and the price of stability \citep{koutsoupias1999, anshelevich2008} are defined as
$$\mathrm{PoA} = \frac{\mu_\theta(S_{\text{opt}})}{\min_{s\in\mathcal N}\mu_\theta(S(s))},\qquad\quad \mathrm{PoS} = \frac{\mu_\theta(S_{\text{opt}})}{\max_{s\in\mathcal N}\mu_\theta(S(s))}.$$
Both the price of anarchy (PoA) and the price of stability (PoS) quantify the cost of decentralization, or lack of cooperation, among agents. Specifically, the PoA is the ratio between the social welfare achieved by an optimal allocation and that achieved at the worst Nash equilibrium, while the PoS uses the best Nash equilibrium instead. 

Since all players opt to play (their expected payout is always strictly positive), situations with many players and few arms are problematic due to arms populating past saturation. It leads to the following unboundedness of both PoA and PoS.

\begin{lem}\label{lem:unbounded_poa}
Both the price of anarchy and price of stability are unbounded for unconstrained instances of the decentralized allocation problem.
\end{lem}
\begin{proof}[Sketch of proof.]
For $N\geq2$, the instance with two arms of means $\theta_1=1$, $\theta_2=4^{-N}$ and $N$ players of activation probability $\frac{1}{2}$ satisfies $\mathrm{PoA}=\mathrm{PoS}\geq\frac{2^{N-1}}{N}$. See \Cref{annexe_unbounded_poa} for a full proof.
\end{proof}
The proof of \Cref{lem:unbounded_poa} constructs decentralized scenarios in which removing the central controller can lead to an arbitrarily large loss in social welfare. However, these scenarios rely on extreme settings in which one arm is much better than all others, and there is enough total player load to push said arm well past saturation.

Instead, we now investigate a constrained but realistic case in which all arm means are in some fixed interval, and therefore players distribute more evenly over available arms. As such, if the total capacity on arms is roughly equal to the total load of players, we can reasonably hope for a bounded PoA.

\begin{thm}\label{thm:poa_bound}
Assume that $\max_i p_i =  q < 1$ and $\theta_k\in[\bar \theta-\varepsilon_1,\bar \theta+\varepsilon_2]$ with $0\leq\varepsilon_1, \varepsilon_2<\bar\theta$ and some $\bar\theta>0$. Then every pure Nash equilibrium $s$ of the allocation game satisfies

$$\mu_{\theta} (S(s)) \geq (1-q)(\bar\theta-\varepsilon_1)Le^{-x}, \qquad\quad \mu_\theta(S_{\text{opt}}) \leq \frac{\bar\theta+\varepsilon_2}{1-q} K h \bigl((1-q)x\bigr)$$
and
$$\mathrm{PoA}\leq \frac{\bar\theta+\varepsilon_2}{\bar\theta-\varepsilon_1}\cdot \begin{cases} \frac{  e^{qx}  } { 1-q  } & \text{ if } x \leq \frac {1}{1-q},\\[3mm]  \frac{  1  }{  (1-q)^{2}  } \cdot \frac{   e^{x-1}   }{  x}   & \text{if } x \geq \frac{1}{1-q}, \end{cases}$$ where $L = \sum_{i\leq N} \frac{p_i}{1-p_i}$ is the total load, $x = \frac{L}{K}$ is the average load and $h(t)=g(\min(t,1))$, for $g:t\mapsto te^{-t}$.
\end{thm}
In particular, \Cref{thm:poa_bound} implies that decentralized selfish policies are asymptotically optimal in mean-field unsaturated regimes with equal arm means; and more generally, they incur a low price of anarchy (and price of stability) when the arms have relatively similar means. 
The complete proof is given in \Cref{annexe_preuve_poa}.

\section{Conclusion}
In this work, we established the NP-hardness of the heterogeneous allocation problem and showed that the allocations produced by the greedy \Cref{alg:greedy} and DORG achieve expected payouts within multiplicative factors of the optimum. While the approximation guarantee for the greedy algorithm is tight, we believe that a more refined analysis could yield improved guarantees for DORG.

Our approximation guarantees provide a foundation for adapting centralized learning algorithms, such as CUCB and potentially the Cautious Greedy algorithm of \citet{richard2024}, to the heterogeneous asynchronous setting.

Extending these results to the decentralized setting remains a major open challenge. Without a central controller, communication is no longer free, and adapting communication protocols such as those introduced by \citet{boursier2019} becomes more challenging in an asynchronous environment. Heterogeneous activation probabilities introduce an additional difficulty: players no longer play symmetric roles in the communication process. Efficient protocols must therefore balance rapid information exchange between frequently active players with the ability to conduct multiple exchanges in parallel. Addressing these challenges is an important direction for future work toward designing efficient learning policies for asynchronous multi-player bandits with heterogeneous players.

\section*{AI Disclosure}

On top of routine assistance with spelling, and more general proofreading of the present paper, LLMs (Claude Opus, used through Claude Code) were used to identify the 3-partition problem as a good candidate for the reduction argument of \Cref{sec:NPhard}. The LLM's attempt at proving said reduction was however incorrect, and the proof found in \Cref{sec:NPhard} was obtained by hand.


\bibliography{main}

@phdthesis{dakdouk,
  title={Massive multi-player multi-armed bandits for internet of things networks},
  author={Dakdouk, Hiba},
  year={2022},
  school={Ecole nationale sup{\'e}rieure Mines-T{\'e}l{\'e}com Atlantique}
}

@inproceedings{anandkumar2010,
  title={Opportunistic spectrum access with multiple users: Learning under competition},
  author={Anandkumar, Animashree and Michael, Nithin and Tang, Ao},
  booktitle={2010 Proceedings IEEE INFOCOM},
  pages={1--9},
  year={2010},
  organization={IEEE}
}

@inproceedings{avner2014,
  title={Concurrent bandits and cognitive radio networks},
  author={Avner, Orly and Mannor, Shie},
  booktitle={Joint European Conference on Machine Learning and Knowledge Discovery in Databases},
  pages={66--81},
  year={2014},
  organization={Springer}
}

@inproceedings{rosenski2016,
  title={Multi-player bandits--a musical chairs approach},
  author={Rosenski, Jonathan and Shamir, Ohad and Szlak, Liran},
  booktitle={International Conference on Machine Learning},
  pages={155--163},
  year={2016},
  organization={PMLR}
}

@inproceedings{besson2018,
  title={Multi-player bandits revisited},
  author={Besson, Lilian and Kaufmann, Emilie},
  booktitle={Algorithmic Learning Theory},
  pages={56--92},
  year={2018},
  organization={PMLR}
}

@inproceedings{bonnefoi2017,
  title={Multi-armed bandit learning in iot networks: Learning helps even in non-stationary settings},
  author={Bonnefoi, R{\'e}mi and Besson, Lilian and Moy, Christophe and Kaufmann, Emilie and Palicot, Jacques},
  booktitle={International Conference on Cognitive Radio Oriented Wireless Networks},
  pages={173--185},
  year={2017},
  organization={Springer}
}

@article{goundan2007,
  title={Revisiting the greedy approach to submodular set function maximization},
  author={Goundan, Pranava R and Schulz, Andreas S},
  journal={Optimization online},
  volume={1},
  year={2007}
}

@article{calinescu2011,
  title={Maximizing a monotone submodular function subject to a matroid constraint},
  author={Calinescu, Gruia and Chekuri, Chandra and Pal, Martin and Vondr{\'a}k, Jan},
  journal={SIAM Journal on Computing},
  volume={40},
  number={6},
  pages={1740--1766},
  year={2011},
  publisher={SIAM}
}

@article{feige2011,
  title={Maximizing non-monotone submodular functions},
  author={Feige, Uriel and Mirrokni, Vahab S and Vondr{\'a}k, Jan},
  journal={SIAM Journal on Computing},
  volume={40},
  number={4},
  pages={1133--1153},
  year={2011},
  publisher={SIAM}
}

@article{buchbinder2015,
  title={A tight linear time (1/2)-approximation for unconstrained submodular maximization},
  author={Buchbinder, Niv and Feldman, Moran and Seffi, Joseph and Schwartz, Roy},
  journal={SIAM Journal on Computing},
  volume={44},
  number={5},
  pages={1384--1402},
  year={2015},
  publisher={SIAM}
}

@article{buchbinder2019,
  title={Constrained submodular maximization via a nonsymmetric technique},
  author={Buchbinder, Niv and Feldman, Moran},
  journal={Mathematics of Operations Research},
  volume={44},
  number={3},
  pages={988--1005},
  year={2019},
  publisher={INFORMS}
}

@inproceedings{gharan2011,
  title={Submodular maximization by simulated annealing},
  author={Gharan, Shayan Oveis and Vondr{\'a}k, Jan},
  booktitle={Proceedings of the twenty-second annual ACM-SIAM symposium on Discrete Algorithms},
  pages={1098--1116},
  year={2011},
  organization={SIAM}
}

@article{avner2019,
  title={Multi-user communication networks: A coordinated multi-armed bandit approach},
  author={Avner, Orly and Mannor, Shie},
  journal={IEEE/ACM Transactions on Networking},
  volume={27},
  number={6},
  pages={2192--2207},
  year={2019},
  publisher={IEEE}
}

@article{boursier2019,
  title={SIC-MMAB: Synchronisation involves communication in multiplayer multi-armed bandits},
  author={Boursier, Etienne and Perchet, Vianney},
  journal={Advances in Neural Information Processing Systems},
  volume={32},
  year={2019}
}

@article{boursier2024,
  title={A survey on multi-player bandits},
  author={Boursier, Etienne and Perchet, Vianney},
  journal={Journal of Machine Learning Research},
  volume={25},
  number={137},
  pages={1--45},
  year={2024}
}

@Inbook{fisher1978,
author="Fisher, M. L.
and Nemhauser, G. L.
and Wolsey, L. A.",
editor="Balinski, M. L.
and Hoffman, A. J.",
title="An analysis of approximations for maximizing submodular set functions---II",
bookTitle="Polyhedral Combinatorics: Dedicated to the memory of D.R. Fulkerson",
year="1978",
publisher="Springer Berlin Heidelberg",
pages="73--87",
}

@article{wangchen2017,
  title={Improving regret bounds for combinatorial semi-bandits with probabilistically triggered arms and its applications},
  author={Wang, Qinshi and Chen, Wei},
  journal={Advances in neural information processing systems},
  volume={30},
  year={2017}
}

@inproceedings{richard2024,
  title={Constant or logarithmic regret in asynchronous multiplayer bandits with limited communication},
  author={Richard, Hugo and Boursier, Etienne and Perchet, Vianney},
  booktitle={International Conference on Artificial Intelligence and Statistics},
  pages={388--396},
  year={2024},
  organization={PMLR}
}

@inproceedings{wang2020,
  title={Optimal algorithms for multiplayer multi-armed bandits},
  author={Wang, Po-An and Proutiere, Alexandre and Ariu, Kaito and Jedra, Yassir and Russo, Alessio},
  booktitle={International Conference on Artificial Intelligence and Statistics},
  pages={4120--4129},
  year={2020},
  organization={PMLR}
}

@inproceedings{buchbinder2024,
  title={Constrained submodular maximization via new bounds for dr-submodular functions},
  author={Buchbinder, Niv and Feldman, Moran},
  booktitle={Proceedings of the 56th annual ACM symposium on theory of computing},
  pages={1820--1831},
  year={2024}
}

@inproceedings{koutsoupias1999,
  title={Worst-case equilibria},
  author={Koutsoupias, Elias and Papadimitriou, Christos},
  booktitle={Annual symposium on theoretical aspects of computer science},
  pages={404--413},
  year={1999},
  organization={Springer}
}

@article{monderer1999,
  title={Potential games},
  author={Monderer, Dov and Shapley, Lloyd S},
  journal={Games and economic behavior},
  volume={14},
  number={1},
  pages={124--143},
  year={1996},
  publisher={Elsevier}
}

@article{liu2010,
  title={Distributed learning in multi-armed bandit with multiple players},
  author={Liu, Keqin and Zhao, Qing},
  journal={IEEE transactions on signal processing},
  volume={58},
  number={11},
  pages={5667--5681},
  year={2010},
  publisher={IEEE}
}

@book{garey1979,
  title={Computers and Intractability: A Guide to the Theory of NP-completeness},
  author={Garey, M.R. and Johnson, D.S.},
  series={Mathematical Sciences Series},
  year={1979},
  publisher={W. H. Freeman}
}

@article{anshelevich2008,
  title={The price of stability for network design with fair cost allocation},
  author={Anshelevich, Elliot and Dasgupta, Anirban and Kleinberg, Jon and Tardos, {\'E}va and Wexler, Tom and Roughgarden, Tim},
  journal={SIAM journal on computing},
  volume={38},
  number={4},
  pages={1602--1623},
  year={2008},
  publisher={SIAM}
}

\newpage
\appendix
\addcontentsline{toc}{section}{Appendix} 
\part{Appendix} 
\parttoc 

\crefalias{section}{appendix}

\section{Restricted monotonicity and submodularity}\label{annexe_preuve_sous_mod}

We now view $\mu_\theta$ as a set function on $[N]\times[K]$, with elements $(i,k)$ being the allocation of player $i$ to arm $k$. If $S$ and $T$ are two allocations, we say $S\subseteq T$ if $S_k\subseteq T_k$ for all $k$.

\begin{lem}\label{lem:restricted_monotone_submod}
Let $\mathcal U$ (resp. $\mathcal V$) be the allocations whose arms all have load at most $1$ (resp. $2$). Then $\mu_\theta$ is monotone on $\mathcal U$ and submodular on $\mathcal V$. Both of these properties break down if their respective conditions are lifted.
\end{lem}

\begin{proof}
Both families are closed under subsets. 
For monotonicity, assume $S \subseteq T$. We add the players of $T\setminus S$ to $S$ sequentially. Each intermediate allocation lies in $\mathcal U$, so all arms have load at most $1$. As such, the marginal gain of each addition is positive, which concludes.

For $S_k\subseteq[N]$, adding player $j$ to $S_k$ changes the marginal gain of adding player $i$ by

\begin{align*}
    &\theta_k p_i \prod_{l \in S_k}(1-p_l)\Bigl[(1-p_j)\Bigl(1-\sum_{l \in S_k}\frac{p_l}{1-p_l}-\tfrac{p_j}{1-p_j}\Bigr)-(1-\sum_{l \in S_k}\frac{p_l}{1-p_l})\Bigr] \\&= -\theta_k p_i p_j \prod_{l \in S_k}(1-p_l)(2-\sum_{l \in S_k}\frac{p_l}{1-p_l}),
\end{align*}
which is negative when $\sum_{l \in S_k}\frac{p_l}{1-p_l}\leq 2$. As such, while loads are all below $2$, we are in a submodular regime.

Reciprocals follow the exact same reasoning: for monotonicity, an intermediate load greater than $1$ means a decrease in payout. For submodularity, the addition of a player past load $2$ increases the marginal of adding another player, losing submodularity.
\end{proof}

\section{Considerations in case of full saturation by greedy}\label{annexe_saturation_complete}
Let $p_1,\dots,p_N$ and $\theta_1,\dots,\theta_K$ be an instance of the allocation problem \eqref{allocation_problem} for which greedy saturates all arms, i.e., $\sum_{i\in S_{\text{Greedy},k}} \frac{p_i}{1-p_i} \geq 1$ for all $k \in [K]$.

\begin{lem}
In instances where greedy saturates all arms, $$\mu_\theta(S_{\text{Greedy}})\geq \frac{1}{e}\sum_{k\leq K}\theta_k$$
\end{lem}

\begin{proof}
The greedy \Cref{alg:greedy} assigns a player to an arm only if the marginal of said addition is positive. As such, for $k\in [K]$, the last player added to arm $k$, say player $j$, verifies that $\theta_k p_j \prod_{i\in S_{\text{Greedy},k}, i\neq j}(1-p_i) \left(1-\sum_{i\in S_{\text{Greedy},k}, i\neq j}\frac{p_i}{1-p_i}\right)>0$.

In particular, this means that $\sum_{i\in S_{\text{Greedy},k}, i\neq j}\frac{p_i}{1-p_i} < 1$. Let $\bar p_j$ be such that $\sum_{i\in S_{\text{Greedy},k}, i\neq j}\frac{p_i}{1-p_i} + \frac{\bar p_j}{1-\bar p_j} = 1$. We create a fictitious $(N+1)$-th player of activation probability $\bar p_j$. Since $\bar p_j \leq p_j$, and given the linearity of the marginal, this in particular means that $f(S_{\text{Greedy},k}) \geq f \bigl((S_{\text{Greedy},k} \setminus \{j\})\cup \{N+1\}\bigr)$

Finally, 
\begin{align*}
    f \bigl( (S_{\text{Greedy},k} \setminus \{j\}) \cup \{N+1\} \bigr) & = (1-\bar p_j)\displaystyle\prod_{i\in S_{\text{Greedy},k}, i\neq j}(1-p_i) \left( \sum_{i\in S_{\text{Greedy},k}, i\neq j}\frac{p_i}{1-p_i} + \frac{\bar p_j}{1-\bar p_j}\right)\\
    &\geq (1-\bar p_j)\displaystyle\prod_{i\in S_{\text{Greedy},k}, i\neq j}(1-p_i)=\frac{1}{1+\frac{\bar p_j}{1-\bar p_j}} \prod_{i\in S_{\text{Greedy},k}, i\neq j} \frac{1}{1+\frac{p_i}{1-p_i}}\\&
    \geq  e^{(-\sum_{i\in S_{\text{Greedy},k}, i\neq j}\frac{p_i}{1-p_i}) - \frac{\bar p_j}{1-\bar p_j}} = \frac{1}{e}
\end{align*}
The result finally comes from summing expected payouts for each arm.
\end{proof}

We now seek an upper bound on $\mu_{\theta}(S_{\text{opt}})$ to conclude.

\begin{lem}\label{lem:borne_sup_sachant_q_sature}
With $q<1$ the maximal activation probability of players, $\mu_{\theta}(S_{\text{opt}})\leq \lceil \frac{1-q}{q}\rceil\frac{q}{1-q}\cdot(1-q)^{\lceil \frac{1-q}{q}\rceil}\cdot\sum_{k \in [K]} \theta_k$
\end{lem}

\begin{proof}
Let $S$ be a set of player indices, and $i\in S$. Conditioning on the activity of player $i$,
\begin{equation*}
    f(S) = (1-p_i)\,f(S\setminus\{i\}) + p_i\prod_{j\in S,\, j\neq i}(1-p_j),
\end{equation*}
so $f(S)$ is an affine function of $p_i$ when the other probabilities are fixed. Its maximum over $p_i\in[0,q]$ is therefore attained at $p_i=0$ or $p_i=q$. Replacing the activation probabilities one by one in this way never decreases $f(S)$, and leads to a set in which every player has activation probability $0$ or $q$.
In particular, an allocation where each arm gets the maximum amount until adding more players is detrimental, i.e., given \eqref{marginal}, exactly 
$\lceil \frac{1-q}{q}\rceil$ players, has higher expected payout than all allocations on $p_1,\dots,p_N$, including optimal ones.

The overall formula follows by adding per-arm payouts.
\end{proof}

Taking the quotient of the two bounds we just obtained finally gives that $$\frac{\mu_{\theta}(S_{\text{Greedy}})}{\mu_{\theta}(S_{\text{opt}})} \geq \frac{1}{\lceil \frac{1-q}{q}\rceil e\frac{q}{1-q}\cdot(1-q)^{\lceil \frac{1-q}{q}\rceil}}$$

This ratio is always at least $e^{-q}$. Indeed, writing $n = \lceil \frac{1-q}{q}\rceil \geq 1$ and using $1-q\leq e^{-q}$,

\begin{align*}
    n\,\frac{q}{1-q}\,(1-q)^{n} = nq\,(1-q)^{n-1} &\leq nq\,e^{-q(n-1)} \\&= e^{q}\cdot nq\,e^{-nq} \\&\leq e^{q-1}
\end{align*}

\section{Per-arm payout inequality for fixed mass}\label{annexe_per-arm_loads}
We wish to formalize that for a fixed arm mass $m = \sum_i p_i$, the expected payout is lowest if all players have infinitesimal activation probability when $m\leq 1$, and is largest when the mass is concentrated on few players of maximal activation probability, yielding closed-form lower and upper bounds for an arm's expected payout.

Throughout this section, for a set $S$ of players, we write
$$f(S) = \sum_{i\in S} p_i \prod_{j \in S, j\neq i}(1-p_j),$$
which coincides with the definition of \Cref{sec:preliminaries} and remains well defined if some player has activation probability $1$.

\begin{lem}
\label{borne_par_bras_masse}
Let $q <1 $ be a fixed real. For every integer $n \geq 1$ and every $n$-tuple $(p_1, \dots, p_n)$ of reals in $[0, q]$ of mass $m = \sum_{i=1}^n p_i$,
we have, with $S = [n]$,
$$f(S) \leq m e^{q-m}$$
and, if furthermore $m \leq 1$,
$$f(S) \geq m e^{-m}$$
\end{lem}
\begin{proof}
For each $i$, since $1-p_j \leq e^{-p_j}$ and $p_i\leq q$,
\begin{align*}
    \prod_{j\in S, j\neq i}(1-p_j) \leq e^{-(m-p_i)} \leq e^{q-m}
\end{align*}
Summing over $i$ with weights $p_i$ gives the upper bound $f(S) \leq me^{q-m}$.

For the lower bound, assume $0 < m\leq 1$. The $\frac{p_i}{m}$ sum to $1$, so the weighted arithmetic-geometric inequality gives
\begin{align*}
    \frac{f(S)}{m} = \sum_{i=1}^n \frac{p_i}{m}\prod_{j\neq i}(1-p_j) & \geq \prod_{i=1}^n \Bigl[\prod_{j\neq i}(1-p_j)\Bigr]^{\frac{p_i}{m}} \\
    & = \prod_{j=1}^n (1-p_j)^{1-\frac{p_j}{m}}.
\end{align*}
Taking logarithms, it remains to show that
$$\sum_{j=1}^n \Bigl(1-\frac{p_j}{m}\Bigr)\ln(1-p_j) \geq -m = -\sum_{j=1}^n p_j$$

Since $m\leq 1$, we have $1-\frac{p_j}{m}\leq 1-p_j$, and since $\ln(1-p_j)\leq 0$,
\begin{align*}
    \Bigl(1-\frac{p_j}{m}\Bigr)\ln(1-p_j) & \geq (1-p_j)\ln(1-p_j) \\ & \geq -p_j
\end{align*}

We conclude by summing all terms.
\end{proof}

\section{Existence of an allocation matching optimal masses}\label{annexe_load_matching}

\begin{lem}\label{greedy_degradation_to_match_opt}
Let $(c_1,\dots,c_K)\in \mathbb R_+^K$, arm means $\theta_1, \dots, \theta_K$ and players $p_1,\dots,p_N$ such that $$\sum _{k\leq K}c_k \leq \sum_{i\leq N} p_i$$
and let $S_{\text{Greedy}}$ be the allocation given by the greedy \Cref{alg:greedy} on some order of players. Assume that greedy assigns every player to an arm.
There exists an allocation $\bar S$, on a set of players potentially different from the initial one, where the activation mass on arm $k$ is exactly $c_k$, and that satisfies $\bar\mu_\theta(\bar S)\leq \mu_\theta(S_{\text{Greedy}})$, with $\bar\mu_\theta$ the payout function for the set of players used in $\bar S$.
\end{lem}

We approach the proof by showing we can shrink players and potentially reassign the removed mass to other arms, while reducing the expected payout.

For a set $S$ of players, we denote $z(S) = \prod_{i\in S}(1-p_i)$ and
$$\Delta(S) = z(S) - f(S)$$
Notice that $p\theta_k \Delta(S_k)$ is exactly  the marginal of adding a player of mass $p$ to arm $k$, as specified in \eqref{marginal}.
\begin{lem}\label{indice}
Let $S$ be a set of players and $p<1$ the activation probability of an additional player of index $i$. Then
\begin{align*}
    f(S\cup\{i\}) & = f(S) + p\Delta(S) \\
    \Delta(S\cup\{i\}) & = (1-p)\Delta(S) - pz(S) \leq \max\{\Delta(S), 0\}.
\end{align*}
In particular, if $S\subseteq T$, then $\Delta(T)\leq \max\{\Delta(S),0\}$.
\end{lem}

\begin{proof}
By conditioning on the activity of the additional player, $f(S\cup\{i\}) = (1-p)f(S) + p\,z(S)$.

Since $z(S\cup\{i\}) = (1-p)z(S)$, we then have
\begin{align*}
\Delta(S\cup\{i\}) &= (1-p)z(S) - (1-p)f(S) - p\,z(S) \\&= (1-p)\Delta(S) - p\,z(S)
\end{align*}
Since $p\,z(S)$ is always strictly positive, we conclude.

The last point is obtained by adding the players of $T\setminus S$ one by one.
\end{proof}

\begin{lem}\label{sequential_degradation}
Let $S_{\text{Greedy}}$ be an allocation obtained by the greedy algorithm on an arbitrary order of the players. We assume that we know the order in which players on a given arm have been allocated. Let $A, B$ be a partition of $[K]$. Let $S$ be an allocation such that
\begin{itemize}
    \item for $k\in A$, $S_k$ is made of the first players assigned to arm $k$ by greedy, the last of which may have a lower activation probability than initially;
    \item for $k'\in B$, $S_{k'}$ contains $S_{\text{Greedy},k'}$.
\end{itemize}

Let $k\in A$, let $p$ be the activation probability in $S$ of the last player of $S_k$, and let $0<b\leq p$. Then the following two actions have the effect of reducing $\mu_\theta(S)$:
\begin{itemize}
    \item Lowering the activation probability of the last player of $S_k$ from $p$ to $p-b$
    \item Lowering the activation probability of the last player of $S_k$ from $p$ to $p-b$, and adding a new player of activation probability $b$ to an arm $k'\in B$
\end{itemize}
Furthermore, the resulting allocation still satisfies the two conditions above. In particular, it is possible to iterate these operations in any order.
\end{lem}

\begin{proof}
Let $i$ be the last player of $S_k$, and $R$ the players of $S_k$ other than $i$. By assumption, $R$ is exactly the set of players assigned to arm $k$ at the moment greedy processed player $i$. Since greedy assigned $i$ to arm $k$, this means that
$$\theta_k \Delta(R) = \max_{\ell\in[K]}\theta_{\ell} \Delta(S^{(i)}_{\ell}) > 0,$$
where $S^{(i)}_{\ell}$ denotes the set of players assigned to arm $\ell$ at that moment.

By \Cref{indice}, $f(R\cup\{i\}) = f(R) + p\,\Delta(R)$, so lowering the activation probability of player $i$ by $b$ adds exactly $-b\theta_k \Delta(R) \leq 0$ to $\mu_\theta(S)$.

We now wish to show that reassigning the removed mass to an arm in $B$ keeps the effect on $\mu_\theta$ negative. Adding a player of activation probability $b$ to arm $k'$ adds $b\theta_{k'}\Delta(S_{k'})$ to $\mu_\theta(S)$,  by \Cref{indice}. Arm $k'$ may have received other players since greedy processed player $i$, but lost none, meaning $S^{(i)}_{k'}\subseteq S_{\text{Greedy},k'}\subseteq S_{k'}$ and, by \Cref{indice},
$$\theta_{k'}\Delta(S_{k'}) \leq \max\{\theta_{k'}\Delta(S^{(i)}_{k'}), 0\} \leq \theta_k \Delta(R) $$

The second action therefore also reduces expected payout.

Finally, arm $k$ still contains the first players assigned to it by greedy, with the last one shrunk or removed, and arm $k'$ only gained a player, so both conditions are preserved.
\end{proof}

\textbf{Proof of \Cref{greedy_degradation_to_match_opt}.}
We are now in a position to prove the main result initially mentioned. Let us start with greedy's allocation, which populates each arm with mass $m_k$, and iterate operations from the above lemma. Let $A$ be the set of arms where $m_k > c_k$, and $B$ its complement. The conditions of \Cref{sequential_degradation} are satisfied by $S_{\text{Greedy}}$ itself.

Since greedy assigns every player, $\sum_k m_k = \sum_i p_i \geq \sum_k c_k$.

While there exists an arm $k$ in $A$ such that $m_k>c_k$, we consider its last player, of activation probability $p$:
\begin{itemize}
    \item if there exists an arm $k'$ in $B$ such that $c_{k'} > m_{k'}$, we move a mass $b = \min\{p,\ m_k - c_k,\ c_{k'}-m_{k'}\}$ from the last player of arm $k$ to a new player on arm $k'$;
    \item otherwise, we remove a mass $b = \min\{p,\ m_k-c_k\}$ from the last player of arm $k$.
\end{itemize}
The masses of the arms in $A$ stay above their associated target, and those of the arms in $B$ stay below.

Each of these ``rounds'' either removes a player entirely from an arm in $A$, or makes an arm exactly match its target mass. As such, there are at most $N+2K$ rounds. Moving mass lowers the excess and the missing mass by the same amount, and mass is only removed when no mass is missing, so the excess mass remains at least the missing mass throughout. When the process stops, there is no excess mass left, hence no missing mass either. We therefore reach the wanted state, all while lowering $\mu_\theta$ in the process.

\section{Calculations giving the lower bound on $\frac{\mu_{\theta}(S_{\text{Greedy}})}{\mu_{\theta}(S_{\text{opt}})}$ }\label{annexe_fin_preuve1/e}
Throughout this section, we assume that the greedy \Cref{alg:greedy} assigns every player to an arm, and denote $q = \max_i p_i$.

Let $m_k^{\text{opt}} = \sum_{i\in S_{\text{opt},k}} p_i$ be the per-arm masses of an optimal allocation, and $(c_1,\dots,c_K) = (\min\{m_1^{\text{opt}},1\},\dots,\min\{m_K^{\text{opt}},1\})$. Since $$\sum_{k\leq K} c_k \leq \sum_{k\leq K} m_k^{\text{opt}} \leq \sum_{i\leq N} p_i,$$ we are in the scope of \Cref{greedy_degradation_to_match_opt}, and therefore have access to an allocation $\bar S$ which shares its masses with an optimal allocation on arms of mass at most $1$, has mass exactly $1$ on the others, and has worse payout than $S_{\text{Greedy}}$.

The lower bound of \Cref{borne_par_bras_masse} gives
$$f(\bar{S}_k) \geq c_k e^{-c_k}$$
while its upper bound, applied to the optimal allocation, gives
$$f(S_{\text{opt},k}) \leq m_k^{\text{opt}} e^{q-m_k^{\text{opt}}}$$

It remains to compare the two. If $m_k^{\text{opt}}\leq 1$, then $c_k = m_k^{\text{opt}}$ and $m_k^{\text{opt}} e^{-m_k^{\text{opt}}} = c_ke^{-c_k}$. Otherwise, $c_k = 1$, and since $t\mapsto te^{-t}$ is upper-bounded by $e^{-1}$ on $\mathbb R_+$, we have $m_k^{\text{opt}} e^{-m_k^{\text{opt}}} \leq e^{-1} = c_ke^{-c_k}$. In both cases,
$$f(S_{\text{opt},k}) \leq e^{q} c_k e^{-c_k} \leq e^{q} f(\bar S_k).$$

Summing across all arms, we finally have
\begin{align*}
    \mu_{\theta}(S_{\text{Greedy}}) \geq \bar\mu_{\theta}(\bar{S}) & = \sum_{k\in[K]}\theta_k f(\bar S_k) \\
    & \geq e^{-q}\sum_{k\in[K]}\theta_k f(S_{\text{opt},k}) \\
    & = e^{-q}\mu_{\theta}(S_{\text{opt}}),
\end{align*}

\section{Proof of tightness of the $\frac{1}{e}$ ratio}\label{annexe_borne_tight}
Given the instance described in the proof sketch, i.e., $\theta_1 = 1$, $\theta_2 = \frac{1}{3N}$, $p_1=\dots=p_{N-1} = \frac{1}{N}$ and $p_N = 1 - \frac{1}{N}$ for some $N\in \mathbb{N}_{\geq 3}$, let us establish what optimal and greedy allocations are:
\begin{itemize}
    \item Optimal allocations assign player $N$ to arm $1$ and the rest to arm $2$, yielding a total expected payout of $1-\frac{1}{N}+\frac{1\cdot(1-\frac{1}{N})^{N-1}}{3N}$.

    \item Greedy processes players in order. The first player is assigned to the better of the two arms, arm $1$. When $k < N-1$ players of activation probability $\frac{1}{N}$ are assigned to arm $1$, the marginal of adding a player of activation probability $p$ is $p(1-\frac{1}{N})^k \frac{N-1-k}{N-1} \geq \frac{p}{3N}$, with the latter being the marginal of adding that same player to the empty arm $2$. As such, all of the $N-1$ first players get assigned to arm $1$. Finally, when $k = N-1$, the marginal of adding a player to arm $1$ becomes $0$, and player $N$ of activation probability $1 - \frac{1}{N}$ gets assigned to arm $2$. The final expected payout is $(1-\frac{1}{N})^{N-1} + \frac{1-\frac{1}{N}}{3N}$
\end{itemize}
Taking the ratio of the two payouts gives $$\frac{\mu_{\theta}(S_{\text{Greedy}})}{\mu_{\theta}(S_{\text{opt}})} = \frac{(1-\frac{1}{N})^{N-1} + \frac{1-\frac{1}{N}}{3N}}{1-\frac{1}{N}+\frac{(1-\frac{1}{N})^{N-1}}{3N}} \xrightarrow[N\rightarrow \infty]{}\frac{1}{e}$$

This concludes.
 \section{Bounds on DORG}\label{annexe_dorg}
We first establish that DORG coincides with at least one optimal allocation on all players with activation probability above $\frac{1}{2}$.
\begin{lem}
Let $\kappa = \max\{i\mid p_i>\frac12\}$, with $\kappa = 0$ if the set is empty. Then, there is an optimal allocation that coincides with DORG for the first $\kappa$ players.
\end{lem}

\begin{proof}
 We first note that for any $q>\frac12$, and a set $A$ of players with activation probabilities $p_j \leq q$, we have $f(A) \leq q$. This is a direct consequence of \Cref{lem:borne_sup_sachant_q_sature}.

 Assuming player $1$, with $p_1> \frac12$, is assigned to some arm $k$, this arm will not receive any supplementary players, as the marginal gain of doing so would be negative. 

Let $S_\text{opt}$ be an optimal allocation, and $k$ be the arm player $1$ is assigned to. Then $S_{\text{opt},k} = \{1\}$, and therefore $f(S_{\text{opt},k}) = p_1$. Since $\max_{i \in S_1} p_i\leq p_1 $,  $f(S_1)\leq f(S_k)$. Swapping $S_k$ and $S_1$ does not decrease the total payout, and therefore an optimal allocation assigning player $1$ to arm $1$ exists. We follow by induction for all other used players with $p_i > \frac12$.

Finally, let us prove that player $1$ is indeed assigned to an arm w.l.o.g. Arm $1$ contains players in $S_1$, which are all of activation probability $p_i \leq p_1$. As such, replacing $S_1$ with $\{1\}$ (i.e., assigning player $1$ instead of the current allocation) while keeping the rest of the allocation unchanged does not decrease the payout. We follow for the next arms by induction, while arms remain. If $\kappa>K$, the remaining players are left unassigned by both DORG and optimal allocations. 
\end{proof}

We therefore know DORG and some optimal allocation $S_{\text{opt}}$ agree on players with activation probability above $\frac{1}{2}$. Since said players are left as sole occupants of their respective arms, $S_{\text{DORG}}$ and $S_{\text{opt}}$ can disagree only on players with activation probability less than $\frac{1}{2}$, and on a subset of arms. As such,

\begin{align*}
\mu_\theta(S_{\text{opt}}) & = \sum_{k\leq K} \theta_kf(S_{\text{opt},k}) \\&= \sum_{k\leq \min\{\kappa, K\} } \theta_k p_k + \sum_{\min\{\kappa,K\} < k \leq K} \theta_k f(S_{\text{opt},k})
\end{align*}

while 
\begin{align*}  
\mu_\theta(S_{\text{DORG}}) &= \sum_{k\leq K} \theta_kf(S_{\text{DORG},k})\\& = \sum_{k\leq \min\{\kappa,K\} } \theta_k p_k + \sum_{\min\{\kappa,K\} < k \leq K} \theta_k f(S_{\text{DORG},k})
\end{align*}
The second term of the sum corresponds to DORG dealing with all players of activation probability below $\frac12$, and as such we have bounds on the ratio of the second terms of both expressions\footnote{Indeed, DORG corresponds to greedy on some fixed order, and as such benefits from instance independent bounds we established before.}, by \Cref{thm_1_e} and \Cref{lem:borne_sup_sachant_q_sature}, as either $\frac{1}{\sqrt e}$ in unsaturated instances, or $\frac2e$ in saturated ones.

As such, we have an instance-independent $\frac{1}{\sqrt e}$ lower bound on $\frac{\mu_\theta(S_{\text{DORG}})}{\mu_\theta(S_{\text{opt}})}$, i.e., DORG is a $\frac{1}{\sqrt e}$-oracle for the allocation problem \eqref{allocation_problem}.

\section{Magnitudes of constants in CUCB's distribution dependent bounds}\label{annexe_cucb_constants}
\citet{wangchen2017} also provide logarithmic bounds that depend on instance dependent constant quantities $\Delta_{\min}$ and $\Delta_{\max}$ : 

$$\text{Reg}_{\theta,\alpha,1}(T) \leq  \sum_{k=1}^{K} \frac{576\,\tilde K \ln T}{\Delta_{\min}^k} \;+\; \sum_{k=1}^{K}\left(\left\lceil \log_2 \frac{2\tilde K}{\Delta_{\min}^k}\right\rceil_0 + 2\right)\frac{\pi^2}{6}\,\Delta_{\max} + 4K $$
with $\tilde K = \min \{K,\,N\}$ the number of pullable arms, for each arm $k$
$$\Delta_{\min}^k = \inf_{S \mid\, f(S_k)>0,\ \Delta_S>0} \Delta_S, \qquad \Delta_{\max} = \max_{k\leq K}\, \sup_{S \mid\, f(S_k)>0,\ \Delta_S>0} \Delta_S$$

and $$\Delta_S = \max\left(0, \, \alpha \mu_\theta(S_{\text{opt}}) - \mu_\theta(S)\right)$$

By definition, $\Delta_{\max} = \alpha\,\mu_\theta(S_{\text{opt}}) - \min\{\mu_\theta(S) \mid S \neq \emptyset\}$. As such, we have the immediate bound
$$\Delta_{\max} \leq \alpha K$$

For $\Delta_{\min}^k$, we seek to find $$r_k = \max\left\{v \in V_k \mid\ v < \alpha\,\mu_\theta(S_{\text{opt}})\right\}$$ where  $V_k = \{\mu_\theta(S) \mid S_k \neq \emptyset\} \subset [0, \mu_\theta(S_{\text{opt}})]$ is the set of payouts for allocations that assign at least one player to arm $k$, and $$\Delta_{\min}^k = \alpha\,\mu_\theta(S_{\text{opt}}) - r_k$$

This quantity is instance-dependent in magnitude, and can be arbitrarily small even in non-degenerate scenarios: for example, for $\alpha = \frac{1}{\sqrt e}$, if $\theta_1 = 1$, $\theta_2 = 1$, and players have activation probabilities $\frac12, \frac12, \frac{1}{\sqrt e}-\frac12-\varepsilon$, then the optimal expected payout is $1$ by splitting the two good players across arms, and having the third player replace one of them on either the first or second arm gives $\Delta_{\min}^1 = \Delta_{\min}^2 = \varepsilon$.

\section{Proof of the existence of Nash equilibria}\label{annexe_jeu_potentiel}
We show that a variant of the allocation game is a potential game.

\begin{thm}\label{thm:potential}
Let $w_i=\ln\frac{1}{1-p_i}$ and $W_k(s)=\sum_{i\in S_k}w_i$, so that $z_k=e^{-W_k(s)}$. Then
$$\Phi(s)=\sum_{k=1}^K\Biggl[\ \sum_{\{i,j\}\subseteq S_k}w_iw_j\;-\;\ln\theta_k\,W_k(s)\ \Biggr]$$
satisfies, for every unilateral deviation $s\to s'$ of a player $j$,
$$\Phi(s')-\Phi(s)=-\,w_j\,\bigl(\ln u_j(s')-\ln u_j(s)\bigr):$$
the ``log-transformed'' game is a weighted potential game.
\end{thm}

\begin{proof}
Since $z_k=e^{-W_k(s)}$, we have $\ln u_j(s)=\ln p_j+w_j+\ln\theta_{s_j}-W_{s_j}(s)$. If $j$ moves from arm $a$ to arm $b$, the pair terms of $\Phi$ change by $w_jW_b-w_j(W_a-w_j)$ and the linear terms by $w_j(\ln\theta_a-\ln\theta_b)$, so that
$$\Phi(s')-\Phi(s)=w_j\bigl[(W_b-\ln\theta_b)-(W_a-w_j-\ln\theta_a)\bigr]=-w_j\bigl(\ln u_j(s')-\ln u_j(s)\bigr).$$
\end{proof}

Consequently, pure equilibria exist, and every better-response path is finite \citep{monderer1999}.

\section{Unbounded PoA and PoS}\label{annexe_unbounded_poa}
Let $N\geq2$. We wish to prove that the instance with two arms of means $\theta_1=1$, $\theta_2=4^{-N}$ and $N$ players of activation probability $\frac{1}{2}$ satisfies $\mathrm{PoA}=\mathrm{PoS}\geq\frac{2^{N-1}}{N}$.

Assigning every player to arm $1$ is an equilibrium, since $\frac{z_1}{1-p_j}=2^{-(N-1)}\geq\theta_2$. 
Its expected payout is $N\cdot2^{-N}$. This is the only pure equilibrium, as any player assigned to arm $2$ could profitably deviate, and all players want to play.

Meanwhile, $\mu_\theta(S_{\text{opt}}) = \frac{1}{2}(1 + 4^{-N})$, assigning a single player to each of the two arms. The ratio of the two concludes.

\section{Proof of PoA bound in \Cref{thm:poa_bound}}\label{annexe_preuve_poa}
We aim to prove the lower bound on $\mu_{\theta}(S(s))$ and the upper bound on $\mu_{\theta}(S_{\text{opt}})$ independently, then take their ratio to reach the final result. We use throughout $l_k = \sum_{i \in S_k}\frac{p_i}{1-p_i}$.
\begin{itemize}
\item

Let $\lambda_k=\theta_k e^{-W_k}$, with $W_k$ the ``log-load''

$$W_k = -\sum_{i \in S_k} \ln(1-p_i)$$

For $p\in(0,q]$, we have 
$$(1-q)\frac{p}{1-p}\leq p\leq\ln\frac1{1-p}\leq\frac{p}{1-p}$$ and therefore $$(1-q)l_k\leq W_k\leq l_k$$

If arm $k$ is occupied by some player $j$, the equilibrium condition imposes $\frac{\lambda_k}{1-p_j}\geq\lambda_{k'}$ for all $k'$, which gives  $\lambda_k\geq(1-q)\max_{k'}\lambda_{k'}$.

Since all players opt to play, $\sum_k l_k=L \geq \sum_k W_k$, so some arm $k'$ has $W_{k'}\leq x$. As such, since $\lambda_k=\theta_k e^{-W_k}$, we have by bounding each term $$\max_{k'}\lambda_{k'} \geq (\bar\theta-\varepsilon_1)e^{-x}$$

Finally, since $\mu_\theta(S)=\sum_k\lambda_kl_k$, the equilibrium payout is at least $$(1-q)(\bar\theta-\varepsilon_1)e^{-x}L$$

\item
For any allocation $S$, by the inequalities above

$$f(S_k)=l_ke^{-W_k}\leq l_ke^{-(1-q)l_k}$$

and since $h \geq g$, 

\begin{align*}
l_ke^{-(1-q)l_k} &= \frac1{1-q}g\bigl((1-q)l_k\bigr)\\&\leq\frac1{1-q}h\bigl((1-q)l_k\bigr)
\end{align*}

By Jensen 
\begin{align*}
\sum_k \theta_k f(S_k)&\leq   K \sum_k   \frac1{K(1-q)}(\bar\theta + \varepsilon_2)h\bigl((1-q)l_k\bigr)\\& \leq  \frac {(\bar\theta + \varepsilon_2)K}{1-q}h \bigl((1-q)\frac{L}{K}\bigr)
\end{align*}

\end{itemize}

The final result is the ratio of the two bounds, in cases where $h = g$ and not.

\end{document}